\documentclass{article}

\usepackage[leqno]{amsmath}
\usepackage{amsthm}
\usepackage{amssymb}
\newtheorem{definition}{Definition}
\newtheorem{proposition}{Proposition}

\newcommand{\R}{\mathbb{R}}
\newcommand{\dn}[1]{\left\downarrow #1\right.}
\newcommand{\dnv}[1]{\left\downarrow_\vdash #1\right.}

\title{Algebras over a field and semantics for context based reasoning}

\author{Daoud Clarke}

\date{\today}

\begin{document}
\maketitle

\section{Introduction}

This chapter introduces context algebras and demonstrates their
application to combining logical and vector-based representations of
meaning. Other chapters in this volume consider approaches that
attempt to reproduce aspects of logical semantics within new
frameworks. The approach we present here is different: We show how
logical semantics can be embedded within a vector space framework, and
use this to combine distributional semantics, in which the meanings of
words are represented as vectors, with logical semantics, in which the
meaning of a sentence is represented as a logical form.

The ideas discussed here are present (at least implicitly) in earlier
work, however we have introduced some notions which allow the
mathematics to be tidied considerably:
\begin{itemize}
\item When context algebras were introduced \cite{Clarke:07} they were
  applied only to functions from a free monoid $A^*$ to $\R$. In fact,
  this construction generalises to functions from $A^*$ to an
  arbitrary vector space $V$. The proof of the general case is
  identical to the specific one, and is reproduced here unchanged.
\item This more general construction gives us an elegant way of
  embedding logical semantics within an algebraic framework. The
  embedding presented here follows similar lines to the thinking of
  \cite{Clarke:07}, but uses the new, more general, context algebras.
\item The method of combining logical semantics with vector-based
  lexical semantics is new, but follows similar lines to an approach
  suggested in \cite{Clarke:11}.
\end{itemize}


\subsection{Motivation}

Like other work in this book, we are concerned with the question of
how to compose vector-based representations of meaning so that phrases
and sentences are also represented as vectors. We wish to preserve the
wonderful flexibility and fine-grained distinctions of meaning that
vector spaces allow, and which have been so successful in lexical
semantics, to build a complete framework for natural language
semantics encompassing words, phrases, sentences and beyond.

Unlike other work, in the approach presented here, we do not attempt
to reconstruct logical semantics from scratch, instead embedding
logical representations within a vector space. This has some
benefits:
\begin{itemize}
\item Doing natural language semantics well is difficult, and a lot of
  work has gone into getting logical semantics for natural language
  right. It includes worrying about things like anaphora resolution,
  generalised quantifiers and negation, and reproducing this work from
  scratch in a vector-based framework is a mammoth task. Our approach
  allows us to reuse existing work while incorporating vector-based
  lexical semantics.
\item There is the potential to reuse existing tools for natural
  language semantics, although computation in general is a problem
  with our approach.
\end{itemize}
The downside to our approach is that we don't yet have an efficient
way of computing with it, although we have ideas for how this may be
achieved. Another potential criticism of this approach is that the
flexibility in how vector representations are combined with logic may
be hindered by requiring the wholesale adoption of existing
formalisms, rather than the more tailored approaches of other work.



\section{Theory of Meaning}
\label{theory}

We first recall some basic definitions:
\begin{definition}[Algebra over a field]
  An algebra over a field is a vector space $\mathcal{A}$ over a field
  $K$ together with a binary operation $(a,b)\mapsto ab$ on
  $\mathcal{A}$ that is bilinear,
\begin{align}
  a(\alpha b + \beta c) &= \alpha ab + \beta ac\\
  (\alpha a+\beta b)c &= \alpha ac + \beta bc
\end{align}
for all $a,b,c\in \mathcal{A}$ and all $\alpha,\beta \in K$. If we
additionally have the property $(ab)c = a(bc)$ then $\mathcal{A}$ is
called \textbf{associative}.  An algebra is called \textbf{unital} if
it has a distinguished \textbf{unity} element $1$ satisfying $1x = x1
= x$ for all $x\in\mathcal{A}$. We are generally only interested in
\textbf{real} associative algebras, where $K$ is the field of real
numbers, $\mathbb{R}$.
\end{definition}

Examples of associative algebras are given by square matrices of order
$n$ under normal matrix multiplication and entry-wise vector
operations. The field of the algebra is the field of the elements of
the matrices; so real valued matrices form a real associative
algebra.

\subsection{Meaning as Context}

The distributional hypothesis of Harris \cite{Harris:68} states that
words will have similar meanings if and only if they occur in similar
contexts. We formalise this idea, and examine the resultant
mathematical properties.

Let $A$ be some set, which we imagine to be the set of words of a
natural language. If $V$ is a vector space, we define a
\textbf{general language for $V$} (or simply a \textbf{language} when
there is no ambiguity) as a function from the free monoid $A^*$ to
$V$. For each string $x \in A^*$, we have associated with it a vector
in $V$ that may have several interpretations:
\begin{itemize}
\item $V$ may simply be the real numbers $\mathbb{R}$, and the
  language may describe a probability distribution over strings in
  $A^*$, in which case we can view the language as a generative model
  of a natural language, describing the probability of observing each
  possible string as a sentence or a document.
\item $V$ may be a vector space describing the meaning of strings, for
  example a representation of model-theoretic semantics. In this case,
  the language attaches a meaning to each possible string in $A^*$.
\end{itemize}

Given a general language $L$ we define the context vector $\hat{x}$ of
a string $x$ as a function from $A^*\times A^*$ to $V$:
$$\hat{x}(y,z) = L(yxz)$$
Thus, as in the study of formal languages, we consider the context of
a string to be the pair of strings surrounding it. We think of
$\hat{x}$ as an element of the vector space $V^{A^*\times A^*}$, the
space of functions from $A^*\times A^*$ to $V$. This is a vector space
with operations defined point-wise, i.e.~if $f,g \in V^{A^*\times
  A^*}$ and $\alpha \in K$ where $K$ is the field of $V$ then $(\alpha
f)(x,y) = \alpha f(x,y)$ and $(f+g)(x,y) = f(x,y) + g(x,y)$ for all
$x,y \in A^*$.

\begin{definition}[Generated Subspace $\mathcal{A}$]
  The subspace $\mathcal{A}$ of $V^{A^*\times A^*}$ is the set defined
  by
\begin{equation}\mathcal{A} = \{a : a = \sum_{x\in A^*}\alpha_x \hat{x}\text{ for some }\alpha_x \in \R\}\end{equation}
In other words, it is the space of all vectors formed from linear combinations
of context vectors.
\end{definition}

Given this definition, we can define multiplication on $\mathcal{A}$,
by assuming linearity, and making the multiplication compatible with
the underlying multiplication of $A^*$. That is, we want to define a
product $\cdot$ on $\mathcal{A}$ such that $\hat{x}\cdot \hat{y} =
\widehat{xy}$ for all $x,y \in A^*$. However, in general there is more
than one basis for $\mathcal{A}$ formed from elements $\hat{x}$, for
$x\in A^*$. We need to confirm that multiplication will be the same,
regardless of which basis we choose.

\begin{proposition}[Context Algebra]
Multiplication on $\mathcal{A}$ is the same irrespective of the choice of basis $B$.
\end{proposition}
\begin{proof}
We say $B \subseteq A^*$ defines a basis $\mathcal{B}$ for $\mathcal{A}$ when $\mathcal{B}$ is a basis such that $\mathcal{B} = \{\hat{x}: x\in B\}$. Assume there are two sets $B_1, B_2 \subseteq A^*$ that define corresponding bases $\mathcal{B}_1$ and $\mathcal{B}_2$ for $\mathcal{A}$. We will show that multiplication in basis $\mathcal{B}_1$ is the same as in the basis $\mathcal{B}_2$.

We represent two basis elements $\hat{u}_1$ and $\hat{u}_2$ of $\mathcal{B}_1$ in terms of basis elements of $\mathcal{B}_2$:
\begin{equation}\hat{u}_1 = \sum_i \alpha_i \hat{v}_i \quad\text{and}\quad
\hat{u}_2 = \sum_j \beta_j \hat{v}_j\end{equation}
for some $u_i \in B_1$, $v_j \in B_2$ and $\alpha_i, \beta_j  \in \R$.
 First consider multiplication in the basis $\mathcal{B}_1$. Note that $\hat{u}_1 = \sum_i \alpha_i \hat{v}_i$ means that $L(xu_1y) = \sum_i \alpha_i L(xv_iy)$ for all $x,y \in A^*$. This includes the special case where $y = u_2y'$ so \begin{equation}L(xu_1u_2y') = \sum_i \alpha_i L(xv_iu_2y')\end{equation} for all $x, y' \in A^*$.
Similarly, we have $L(xu_2y) = \sum_j \beta_j L(xv_jy)$ for all $x,y \in A^*$ which includes the special case $x = x'v_i$, so $L(x'v_iu_2y) = \sum_j \beta_j L(x'v_iv_jy)$ for all $x',y \in A^*$. Inserting this into the above expression yields
\begin{equation}L(xu_1u_2y) = \sum_{i,j} \alpha_i\beta_j L(xv_iv_jy)\end{equation}
for all $x,y \in A^*$ which we can rewrite as
\begin{equation}\hat{u}_1\cdot\hat{u}_2 = \widehat{u_1u_2} = \sum_{i,j}\alpha_i\beta_j (\hat{v}_i\cdot\hat{v}_j)
= \sum_{i,j}\alpha_i\beta_j \widehat{v_iv_j}\end{equation}
Conversely, the product of $u_1$ and $u_2$ using the basis $\mathcal{B}_2$ is
\begin{equation}\hat{u}_1\cdot \hat{u}_2 = \sum_i \alpha_i \hat{v}_i \cdot \sum_j \beta_j \hat{v}_j =  \sum_{i,j}\alpha_i\beta_j (\hat{v}_i\cdot\hat{v}_j)\end{equation}
thus showing that multiplication is defined independently of what we choose as the basis.
\end{proof}

Multiplication as defined above makes $\mathcal{A}$ an algebra,
moreover it is easy to see that it is associative since the
multiplication on $A^*$ is associative. It has a unity, which is given
by $\hat{\epsilon}$, where $\epsilon$ is the empty string.

\subsection{Entailment}

Our notion of entailment is founded on the idea of
\textbf{distributional generality} \cite{Weeds:04}. This is the idea
that the distribution over contexts has implications not only for
similariy of meaning, but can also describe how general a meaning
is. A term $t_1$ is considered distributionally more general than
another term $t_2$ if $t_1$ occurs in a wider range of contexts than
$t_2$. It is proposed that distributional generality may be connected
to semantic generality. For example, we may expect the
term \emph{animal} to occur in a winder range of contexts than the
term \emph{cat} since the first is semantically more general.

We translate this to a mathematical definition by making use of an
implicit partial ordering on the vector space:

\begin{definition}[Partially ordered vector space]\index{vector space!partially ordered|textbf}
A partially ordered vector space $V$ is a real vector space together with a partial ordering $\le$ such that:
\vspace{0.1cm}\\
\indent if $x \le y$ then $x + z \le y + z$\\
\indent if $x \le y$ then $\alpha x \le \alpha y$
\vspace{0.1cm}\\
for all $x,y,z \in V$, and for all $\alpha \ge 0$. Such a partial ordering is called a \textbf{vector space order} on $V$. An element $u$ of $V$ satisfying $u \ge 0$ is called a \textbf{positive element}; the set of all positive elements of $V$ is denoted $V^+$. If $\le$ defines a lattice on $V$ then the space is called a \textbf{vector lattice} or \textbf{Riesz space}.
\end{definition}

If $V$ is a vector lattice, then the vector space of contexts,
$V^{A^*\times A^*}$, is a vector lattice, where the lattice operations
are defined component-wise: $(u\land v)(x,y) = u(x,y)\land v(x,y)$,
and $(u\lor v)(x,y) = u(x,y)\lor v(x,y)$. For example, $\R$ is a
vector lattice with meet as the $\min$ operation and join as $\max$,
so $\R^{A^*\times A^*}$ is also a vector lattice. In this case, where
the value attached to a context is an indication of its frequency of
occurrence, $\hat{x} \le \hat{y}$ means that $y$ occurs at least as
frequently as $x$ in every context.

Note that, unlike the vector operations, the lattice operations are
dependent on the basis: a different basis gives different
operations. This makes sense in the linguistic setting, since there is
nearly always a distinguished basis, originating in the contexts from
which the vector space is formed.

\section{From logical forms to algebra}
\label{applications}

Model-theoretic approaches generally deal with a subset of all
possible strings, the language under consideration, translating
sequences in the language to a logical form, expressed in another,
logical language. Relationships between logical forms are expressed by
an entailment relation on this logical language.

This section is about the algebraic representation of logical
languages. Representing logical languages in terms of an algebra will
allow us to incorporate statistical information about language into
the representations. For example, if we have multiple parses for a
sentence, each with a certain probability, we will be able to
represent the meaning of the sentence as a probabilistic sum of the
representations of its individual parses.

By a logical language we mean a language $\Lambda
\subseteq A'^*$ for some alphabet $A'$, together with a relation
$\vdash$ on $\Lambda$ that is reflexive and transitive; this relation
is interpreted as entailment on the logical language. We will show how
each element $u \in \Lambda$ can be associated with a projection on a
vector space; it is these projections that define the algebra. Later
we will show how this can be related to strings in the natural
language $\lambda$ that we are interested in.

For a subset $T$ of a set $S$, we define the projection $P_T$ on
$L^\infty(S)$ (the set of all bounded real-valued functions on $S$) by
$$P_T e_s = \left\{
\begin{array}{ll}
e_s & \text{if } s \in T\\
0 & \text{otherwise}
\end{array}
\right.$$ 
Where $e_s$ is the basis element of $L^\infty(S)$
corresponding to the element $s \in S$. Given $u \in \Lambda$ , define
$\left\downarrow_\vdash(u)\right. = \{v : v \vdash u\}$.

As a shorthand we write $P_u$ for the projection $P_{\left\downarrow_\vdash(u)\right.}$ on the space
$L^\infty(\Lambda)$.  The projection $P_u$ can be thought of as projecting onto the
space of logical statements that entail u. This is made formal in the
following proposition:

\begin{proposition}
\label{projections}
$P_u \le P_v$ if and only if $u \vdash v$.
\end{proposition}
\begin{proof}
  Recall that the partial ordering on projections is defined by $P_u
  \le P_v$ if and only if $P_uP_v = P_vP_u = P_u$
  \cite{Aliprantis:85}. Clearly
$$P_uP_v e_w = \left\{
\begin{array}{ll}
e_w & \text{if } w \vdash u \text{ and } w \vdash v\\
0 & \text{otherwise}
\end{array}
\right.$$
so if $u\vdash v$ then since $\vdash$ is transitive, if $w\vdash u$
then $w\vdash v$ so we must have $P_uP_v = P_vP_u = P_u$.

Conversely, if $P_u P_v = P_u$ then it must be the case that $w \vdash u$ implies $w \vdash v$ for
all $w \in \Lambda$, including $w = u$. Since $\vdash$ is reflexive, we have $u \vdash u$, so $u \vdash v$ which
completes the proof.
\end{proof}

To help us understand this representation better, we will show that it
is closely connected to the ideal completion of partial orders. Define
a relation $\equiv$ on $\Lambda$ by $u \equiv v$ if and only if $u
\vdash v$ and $v \vdash u$. Clearly $\equiv$ is an equivalence
relation; we denote the equivalence class of $u$ by $[u]$. Equivalence
classes are then partially ordered by $[u] \le [v]$ if and only if $u
\vdash v$. Then note that $\bigcup\dnv{([u])} = \dnv{(u)}$, thus $P_u$
projects onto the space generated by the basis vectors corresponding
to the elements $\bigcup\dnv{([u])}$ , the ideal completion
representation of the partially ordered equivalence classes.

What we have shown here is that logical forms can be viewed as
projections on a vector space. Since projections are operators on a
vector space, they are themselves vectors; viewing logical
representations in this way allows us to treat them as vectors, and we
have all the flexibility that comes with vector spaces: we can add
them, subtract them and multiply them by scalars; since the vector
space is also a vector lattice, we also have the lattice operations of
meet and join. As we will see in the next section, in some special
cases such as that of the propositional calculus, the lattice meet and
join coincide with logical conjunction and disjunction.

\subsection{Example: Propositional Calculus}

In this section we apply the ideas of the previous section to an
important special case: that of the propositional calculus. We choose
as our logical language $\Lambda$ the language of a propositional
calculus with the usual connectives $\lor$, $\land$ and $\neg$, the logical constants
$\top$ and $\bot$ representing ``true'' and ``false'' respectively,
with $u \vdash v$ meaning ``infer $v$ from $u$'', behaving in the usual
way. Then:
\begin{eqnarray*}
P_{u\land v} &=& P_u P_v\\
P_{\neg u} &=& 1 - P_u + P_\bot\\
P_{u\lor v} &=& P_u + P_v - P_u P_v\\
P_\top &=& 1
\end{eqnarray*}
To see this, note that the equivalence classes of $\vdash$ form a Boolean algebra under
the partial ordering induced by $\vdash$, with
$$[u \land v] = [u] \land [v]$$
$$[u \lor v] = [u] \lor [v]$$
$$[\neg u] = \neg[u].$$
Note that while the symbols $\land$, $\lor$ and $\neg$ refer to
logical operations on the left hand side, on the right hand side they
are the operations of the Boolean algebra of equivalence classes; they
are completely determined by the partial ordering associated with
$\vdash$.\footnote{In the context of model theory, the Boolean algebra
  of equivalence classes of sentences of some theory T is called the
  Lindenbaum-Tarski algebra of T \cite{Hinman:05}.  }

Since the partial ordering carries over to the ideal completion we
must have
\begin{eqnarray*}
\dn{[u \land v]} &=& \dn{[u]} \cap \dn{[v]}\\
\dn{[u \lor v]} &=& \dn{[u]} \cup \dn{[v]}
\end{eqnarray*}
Since $u \vdash \top$ for all $u \in \Lambda$, it must be the case
that $\dn{[\top]}$ contains all sets in the ideal completion. However the
Boolean algebra of subsets in the ideal completion is larger than the
Boolean algebra of equivalence classes; the latter is embedded as a
Boolean sub-algebra of the former. Specifically, the least element in
the completion is the empty set, whereas the least element in the
equivalence class is represented as $\dn{[\bot]}$ . Thus negation carries
over with respect to this least element:
$$[\neg u] = ([\top] - [u]) \cup [\bot].$$

We are now in a position to prove the original statements:
\begin{itemize}
\item Since $\dn{[\top]}$ contains all sets in the completion,
  $\bigcup \dn{[\top]} = \left\downarrow_\vdash (\top)\right. =
  \Lambda$, and $P_\top$ must project onto the whole space, that is
  $P_\top = 1$.
\item Using the above expression for $\dn{[u \land v]}$, taking unions of the disjoint sets in
the equivalence classes we have $\left\downarrow_\vdash (u \land v)\right. = \dnv{(u)} \cap \dnv{(v)}$. Making use of
the equation in the proof to Proposition \ref{projections}, we have
$P_{u\land v} = P_u P_v$ .
\item In the above expression for $\dn{[\neg u]}$, note that
  $\dn{[\top]} \subseteq \dn{[u]} \subseteq \dn{[\bot]}$ . This allows
  us to write, after taking unions and converting to projections,
  $P_{\neg u} = 1 - P_u + P_\bot$ , since $P_\top = 1$.
\item Finally, we know that $u \lor v \equiv \neg(\neg u \land \neg v)$, and since equivalent elements in $\Lambda$
have the same projections we have
\begin{eqnarray*}
P_{u\lor v} &=& 1 - P_{\neg u\land\neg v} + P_\bot\\
&=& 1 - P_{\neg u} P_{\neg v} + P_\bot\\
&=& 1 - (1 - P_u + P_\bot )(1 - P_v + P_\bot ) + P_\bot\\
&=& P_u + P_v - P_u P_v - 2P_\bot + P_\bot P_u + P_\bot P_v\\
&=& P_u + P_v - P_u P_v
\end{eqnarray*}
\end{itemize}
It is also worth noting that in terms of the vector lattice operations $\lor$ and $\land$ on
the space of operators on $L^\infty(\Lambda)$, we have $P_{u\lor v} =
P_u \lor P_v$ and $P_u\land v = P_u \land P_v$.

\subsection{From logic to context algebras}



In the simplest case, we may be able to assign to each natural
language sentence a single sentence in the logical language (its
interpretation). Let $\Gamma \subseteq A^*$ be a formal language
consisting of natural language sentences $x$ with a corresponding
interpretation in the logical language $\Lambda$, which we denote
$\rho(x)$. The function $\rho$ maps natural language sentences to
their interpretations, and may incorporate tasks such as word-sense
disambiguation, anaphora resolution and semantic disambiguation.

We can now define a general language to represent this situation. We
take as our vector space $V$ the space generated by projections $\{P_u
: u \in \Lambda\}$ on $L^\infty(\Lambda)$. For $x\in A^*$ we define
$$L(x) = \left\{
\begin{array}{ll}
P_{\rho(x)} & \text{if } x \in \Gamma\\
0 & \text{otherwise}
\end{array}
\right.$$ 
Given the discussions in the preceding sections, it is clear that for
$x,y\in \Gamma$, $L(x) \le L(y)$ if and only if $\rho(x) \vdash
\rho(y)$, so the partial ordering of the vector space encodes the
entailment relation of the logical language.

The context algebra constructed from $L$ gives meaning to any
substring of elements of $\Gamma$, so any natural language expression
which is a substring of a sentence with a logical interpretation has a
corresponding non-zero element in the algebra. If $\Gamma$ has the
property that no element of $\Gamma$ is a substring of any other
element of $\Gamma$ (for example, $\Gamma$ consists of natural
language sentences starting and ending with a unique symbol), then we
will also have $\hat{x} \le \hat{y}$ if and only if $\rho(x) \vdash
\rho(y)$, for all $x,y \in \Gamma$. In this case, the only context
which maps to a non-zero vector is the pair of empty strings,
$(\epsilon, \epsilon)$.

\subsection{Incorporating Word Vectors}

The construction of the preceding section is not very useful on its
own, as it merely encodes the logical reasoning within a vector space
framework. However, we will show how this construction can be used to
incorporate vector-based representations of lexical semantics.

In the general case, we assume that associated with each word $a \in
A$ there is a vector $\psi(a)$ in some finite-dimensional vector space
$L^\infty(S)$ that represents its lexical semantics, perhaps obtained
using latent semantic analysis or some other technique. $S$ is a
finite set indexing the lexical semantic vector space which we
interpret as containing \textbf{aspects} of meaning. This set may be
partitioned into subsets containing aspects for different parts of
speech, for example, we may wish that the vector for a verb is always
disjoint to the vector for a noun. Similarly, there may be a single
aspect for each closed class or function word in the natural language,
to allow these words not to have a vector nature.

Instead of mapping directly from strings of the natural language to
the logical language, we map from strings of aspects of meaning. Let
$\Delta \subset S^*$ be a formal language consisting of all meaningful
strings of aspects, i.e.~those with a logical interpretation. As
before, we assume a function $\rho$ from $\Delta$ to a logical
language $\Lambda$. The corresponding context algebra $\mathcal{A}$
describes composition of aspects of meaning.

We can now describe the representation of the meaning $\tilde{a}$ of a
word $a\in A$ in terms of elements of $\mathcal{A}$:
$$\tilde{a} = \sum_{s \in S} \psi(a)_s\hat{s} $$
thus a term is represented as a weighted sum of the context vectors
for its aspects. Composition of $\mathcal{A}$ together with the
distributivity of the algebra is then enough to define vectors for any
string in $A^*$. For $x \in A^*$, we define $\tilde{x}$ by
$$\tilde{x} = \tilde{x}_1 \cdot \tilde{x}_2 \cdots \tilde{x}_n$$
where $x = x_1x_2\cdots x_n$ for $x_n \in A$.

This construction achieves the goal of combining vector space
representations of lexical semantics with existing logical
formalisms. It has the following properties:
\begin{itemize}
\item The entailment relation between the logical expressions
  associated with sentences is encoded in the partial ordering of the
  vector space.
\item The vector space representation of a sentence includes a sum
  over all possible logical sentences, where each word has been
  represented by one of its aspects.
\end{itemize}

\subsubsection*{Discussion}

This second property is actually inconsistent with the idea that
distributional generality determines semantic generality. To see why,
consider the sentences
\begin{enumerate}
\item \emph{No animal likes cheese}
\item \emph{No cat likes cheese}
\end{enumerate}
Although the first sentence entails the second, the term \emph{animal}
is more general than \emph{cat}; the quantifier \emph{no} has reversed
the direction of entailment. The distributivity of the algebra means
that if the term \emph{animal} occurs in a wider range of contexts
than \emph{cat} (i.e.~it is distributionally more general), then this
generality must persist through to the sentence level as terms are
multiplied.

There are two ways of viewing this inconsistency:
\begin{enumerate}
\item The construction is incorrect because the distributional
  hypothesis does not hold universally at the lexical level. In fact,
  this idea is justified by our example above, in which the more
  general term, \emph{animal}, would be expected to occur in a smaller
  range of contexts than \emph{cat} when preceded by the quantifier
  \emph{no}. Under this view, the construction needs to be altered so
  that quantifiers such as \emph{no} can reverse the direction of
  entailment. In fact, for this to be the case, we would need to
  dispense with using an algebra altogether, since distributivity is a
  fundamental property of the algebra.
\item The construction is correct as long as the ``aspects'' of words
  are really their senses, so the vector space nature only represents
  semantic ambiguity and not semantic generality. In reality, aspects
  may be more subtle than what is normally considered a word sense,
  and the vector space representation would capture this subtlety. In
  this view, distributional generality is correlated with semantic
  generality at the sentence level, but not necessarily the word
  level. Vector representations describe only semantic ambiguity and
  not generality, which should be incorporated into the associated
  logical representation. This would mean that the algorithms used for
  obtaining vectors for terms would have to have more in common with
  automatic word sense induction than the more general semantic
  induction associated with the distributional hypothesis.
\end{enumerate}

An alternative solution for this type of quantifier (another example
is \emph{all}) and negation in general, is to use a construction
similar to that proposed in \cite{Preller:09}, which uses Bell states
to swap dimensions in a similar manner to qubit operators.

\subsection{Partial Entailment}

In general, it is unlikely that any two strings $x,y \in A^*$ which
humans would judge as entailing would have vectors such that
$\tilde{x} \le \tilde{y}$ because of the nature of the automatically
obtained vectors. Instead, it makes sense to consider a \emph{degree}
of entailment between strings. In this case, we assume we have a
linear functional $\phi$ on $V^{A^*\times A^*}$. The degree to which
$x$ entails $y$ is then given defined as
$$\frac{\phi(\tilde{x} \land \tilde{y})}{\phi(\tilde{x})}$$
This has many of the properties we would expect from a degree of
entailment. In certain cases $\phi$ can be used to make the space an
Abstract Lebesgue space \cite{Abramovich:02}, in which case we can
interpret the above definition as a conditional probability. This
idea, and methods of defining $\phi$ are discussed in detail in
\cite{Clarke:07}.

For the vector space $V$ generated by projections on
$L^\infty(\Lambda)$, one possible definition of $\phi$ would be given
by
$$\phi(u) = \sum_{l\in\Lambda} P(l)\|u(\epsilon,\epsilon)e_l\|$$
where $P(l)$ is some probability distribution over elements of
$\Lambda$.  The interpretation of this is: all contexts except those
consisting of a pair of empty strings are ignored (so only strings
that have logical interpretations make a contribution to the value of
the linear functional); the value is then given by summing over all
strings of the logical language and multiplying the probability of
each string $l$ by the size of the vector resulting from the action of the
operator on the basis element corresponding to $l$. The probability
distribution over $\Lambda$ needs to be estimated, this could perhaps
be done using machine learning:
\begin{itemize}
\item Given a corpus, build a model for each sentence in the corpus
  and its negation using a model builder
\item Train a support vector machine using the models for the sentence
  and its negation as the two classes. This requires defining a kernel
  on models.
\item Given a new string $l \in \Lambda$, build a model, then use the
  probability estimate of the support vector machine for the model
  belonging to the positive class, together with a normalisation
  function.
\end{itemize}

\subsection{Computational Issues}

In general it will be very hard to compute the degree of entailment
between two strings using the preceding definitions. The number of
logical interpretations that need to be considered increases
exponentially with the length of the string. One possible way of
tackling this would be to use Monte-Carlo techniques, for example,
sampling dimensions of the vector space and computing degrees of
entailment for the sample.

A more principled approach would be to exploit symmetries in the
mapping $\rho$ from strings of aspects to the logical language. A
further possibility is that a deeper analysis of the algebraic
properties of context algebras leads to a simpler method of
computation. Further work is undoubtedly necessary to tackle this
problem.

\section{Conclusion}
\label{conclusion}

We have introduced a more general definition of context algebras, and
shown how they can be used to combine vector-based lexical semantics
with logic based semantics at the sentence level. Whilst computational
issues remain to be resolved, our approach allows the reuse of the
abundance of work in logic-based natural language semantics.

\bibliographystyle{plain}
\bibliography{contexts}

\end{document}